\documentclass{article}

\usepackage{microtype}
\usepackage{graphicx}
\usepackage{subcaption}
\usepackage{booktabs} 
\usepackage{hyperref}

\usepackage[preprint]{icml2026}

\usepackage{amsmath}
\usepackage{amssymb}
\usepackage{mathtools}
\usepackage{amsthm}

\usepackage[capitalize,noabbrev]{cleveref}

\theoremstyle{plain}
\newtheorem{theorem}{Theorem}[section]

\theoremstyle{definition}

\theoremstyle{remark}

\usepackage[textsize=tiny]{todonotes}

\icmltitlerunning{Position: The Inevitable End of One-Architecture-Fits-All-Domains in Time Series Forecasting}

\begin{document}

\twocolumn[
  \icmltitle{Position: The Inevitable End of\\One-Architecture-Fits-All-Domains in Time Series Forecasting}
  \icmlsetsymbol{equal}{*}

  \begin{icmlauthorlist}
    \icmlauthor{Qinwei Ma}{equal,thu}
    \icmlauthor{Jingzhe Shi}{equal,thu}
    \icmlauthor{Jiahao Qiu}{pton}
    \icmlauthor{Zaiwen Yang}{thu}
  \end{icmlauthorlist}

  \icmlaffiliation{thu}{Tsinghua University}
  \icmlaffiliation{pton}{Princeton University}

  \icmlcorrespondingauthor{Qinwei Ma}{qinweimartin@gmail.com}
  \icmlcorrespondingauthor{Jingzhe Shi}{shi-jz21@mails.tsinghua.edu.cn}

  \icmlkeywords{Machine Learning, ICML}

  \vskip 0.3in
]



\printAffiliationsAndNotice{*\,Equal contribution and equal correspondence.}

\begin{abstract}
    Recent work has questioned the effectiveness and robustness of neural network architectures for time series forecasting tasks. We summarize these concerns and analyze groundly their \textbf{inherent limitations}: i.e. the \textbf{irreconcilable conflict} between single (or few similar) domains SOTA and generalizability over general domains for time series forecasting neural network architecture designs. Moreover, neural networks architectures for general domain time series forecasting are becoming more and more complicated and their performance has almost saturated in recent years. As a result, network architectures developed aiming at fitting general time series domains are almost not inspiring for real world practices for certain single (or few similar) domains such as Finance, Weather, Traffic, etc: each specific domain develops their own methods that rarely utilize advances in neural network architectures of time series community in recent 2-3 years. As a result, we call for the time series community to shift focus away from research on time series neural network architectures for general domains: these researches have become saturated and away from domain-specific SOTAs over time. We should either (1) focus on deep learning methods for certain specific domain(s), or (2) turn to the development of meta-learning methods for general domains.
\end{abstract}

\section{Introduction} 

Time series forecasting is a cornerstone of decision-making in high-stakes environments, including financial markets, energy grid management, climate science, and healthcare. Driven by the transformative success of large-scale pre-training and universal architectures in Natural Language Processing (NLP) and Computer Vision (CV), the time series research community has recently pivoted toward a similar paradigm~\cite{foundationalmodelssurvey}. The prevailing trend assumes that increasing model capacity and designing complex, ``general-purpose'' neural architectures will eventually yield a foundation model for forecasting~\cite{unitsfoundationalmodel,googletimeseriesfoundationalmodel}.

However, we argue that this pursuit overlooks a fundamental structural difference between time series and other common domains. Despite the proliferation of sophisticated architectures, a significant gap remains between academic benchmarks and real-world deployment: when state-of-the-art (SOTA) performance is required for specific domains, practitioners consistently favor domain-specific models over the ``universal'' architectures promoted in the literature. This suggests that the current research trajectory may be reaching a point of diminishing returns~\cite{nochampions,wang2025accuracylawfuturedeep}.

\textbf{The central thesis of this paper is that the community faces a fundamental conflict of \textit{domain-specific excellence} and \textit{cross-domain generalizability}. Moreover, such conflict is irreconcilable for general-domain time series forecasting.} This tension is rooted in two primary constraints:

\begin{enumerate}
    \item \textbf{Inherent Domain Heterogeneity:} Unlike in NLP where each domain shares similar linguistic structures, time series data across different sectors, such as wind power fluctuations versus high-frequency limit order books, are governed by entirely different physical processes and causal structures. A search for a single architecture that is capable of capturing these disparate ``logics'' optimally is likely a misplaced endeavor~\cite{invitedtalkmultimodaltsm}.
    \item \textbf{Statistical Limits of Temporal Data:} Time series data is intrinsically constrained by the temporal window of observation. While CV and NLP benefit from nearly infinite spatial or linguistic scaling, time series scaling is bounded by history. Theoretically, for a data duration $T$, the generalization error is lower-bounded by a factor proportional to $1/\sqrt{T}$. No architectural innovation can circumvent this statistical bottleneck~\cite{kuznetsov2014generalization,ScalingLawTSF2024}.
\end{enumerate}

\begin{figure*}[!h]
  \centering
  \includegraphics[width=0.95\textwidth]{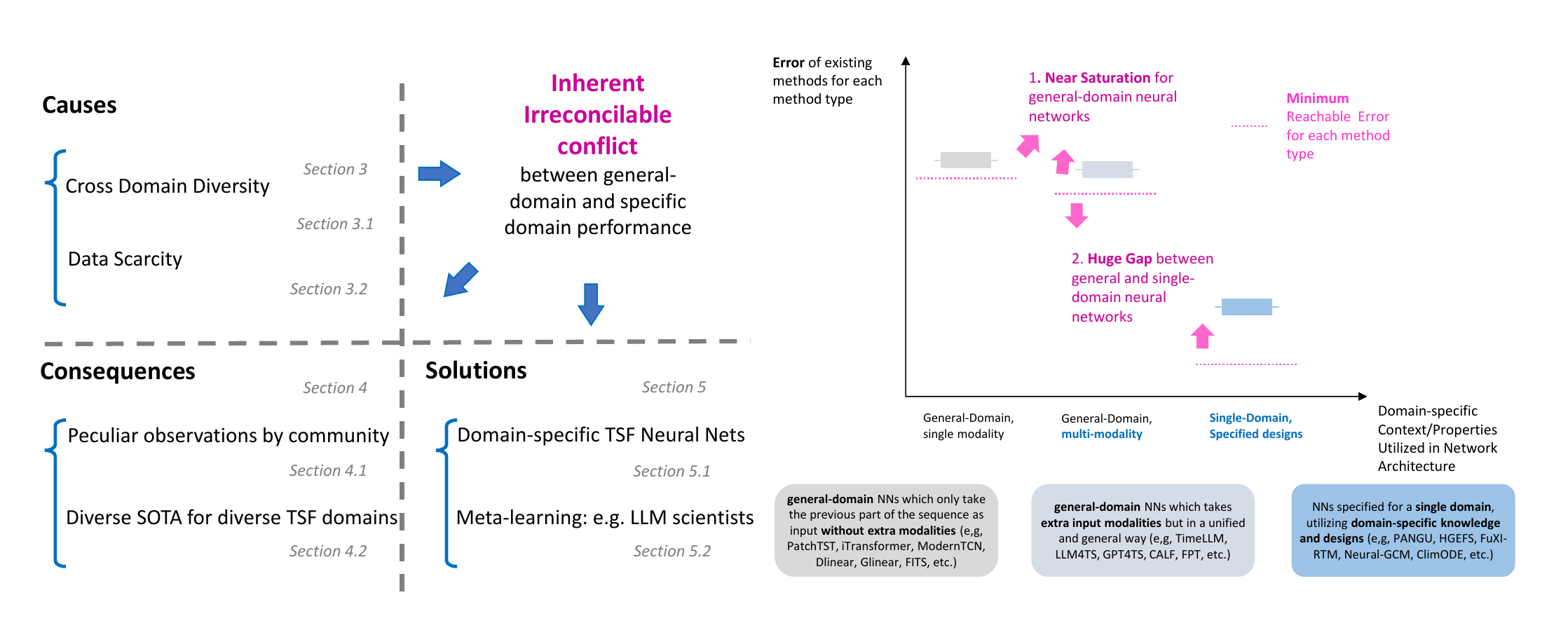}
  \caption{\textbf{Left:} A mind map of the core arguments in this position paper. We propose that there exists a inherent, irreconcilable conflict between general-domain and single-domain performance for time series forecasting neural network architecture designing. We analyze its causes, consequences and advocated better research directions. We discuss more about Alternative Views in Section~\ref{sec: alternative views}. \textbf{Right:} Error reached/reachable for a specific domain (e.g. weather forecasting) when utilizing different types of methods with different available domain-specific context and properties.}
  \label{fig:intro_figure}
\end{figure*}

We argue that it is such conflict that has caused many peculiar phenomena observed in time series communities~\cite{neurips24talk,nochampions}, and has been the reason why advances in general-domain TSF neural networks barely benefits researchers in those specific domains (e.g. Traffic, Weather, Electricity, etc.). In this paper, we advocate for a strategic redirection of the field. As illustrated in Figure~\ref{fig:intro_figure}, we summarize our core arguments and the resulting research redirection. Specifically, we propose that:

\begin{itemize}
    \item The community should acknowledge the ``seesaw'' conflict between domain SOTA and cross-domain generalization. General-purpose architectures that are ``SOTA'' on benchmark suites still lag domain-specific models; given that practitioners can often deploy domain SOTA directly, we should stop optimizing general-purpose architectures as an end in itself.
    \item Research effort should shift to either (1) online meta-learning and continual adaptation that enable models to ``learn how to learn'' from limited domain context (e.g., LLM-as-scientist), or (2) domain-specific forecasting models that explicitly incorporate domain context and mechanisms (finance, weather, traffic, etc.).
\end{itemize}

By pivoting from architecture search to meta-adaptation, our community can bridge the gap between theoretical research and the practical necessity of domain-specific precision.

\section{Background} 

Since 2021, the research community has made great effort in developing general domain time series forecasting neural network architectures. Starting from paradigm and datasets proposed in Informer~\cite{Informer2021}, many neural networks architectures have been proposed for general domain time series forecasting performance. A common practice for these papers is to propose a neural network architecture to achieve as high performance as possible across a wide range of datasets of time series like Electricity, Traffic, Weather, etc.

\subsection{Neural Network Methods for General-domain Time Series Forecasting}

In 2021, Informer~\cite{Informer2021} set up a multi-dataset paradigm (e.g. ETT, ECL, etc.) for neural network based general domain long-horizon forecasting, after which a burst of ``general-domain time series forecasting'' architectures were proposed under a mostly unified task protocol. Beyond Transformer variants (decomposition, frequency, patching, or de-stationary attention)~\cite{Autoformer2022,FEDformer2022,ETSformer2022,NonStationaryTransformers2022,itransformer,TimesNet2022,PatchTST2022,Triformer2022,Spacetimeformer2021}, the community also explored competitive non-attention backbones, including lightweight linear variants (e.g. FITS, OLinear, CrossLinear, GLinear)~\cite{FITS2024,OLinear2025,CrossLinear2025,GLinear2025}, modern CNN-style sequence models (e.g. ModernTCN, PatchMixer)~\cite{moderntcn,PatchMixer2023}, and MLP-mixer style models (e.g. TSMixer, TimeMixer)~\cite{TSMixer2023,TimeMixer2024}.

Crucially, some early works already exposed warning signs about the ``general-domain'' time series. DLinear~\cite{DLinear2022} shows that, under the same benchmark protocol, a single linear layer can match or nearly match many increasingly elaborate architectures; previous context-length scaling analyses~\cite{ScalingLawTSF2024} further suggest that simply extending the history window can even 
\emph{hurt} performance on small datasets. These pieces of work in previous year also questioned the effectiveness of proposed methods and metrics.

\subsection{Large Multi-Domain Time Series Datasets and Large Foundational Time Series Models}

\textbf{Large Datasets.} A common practice is to propose large datasets consisted of several small datasets, represented by recent benchmarking efforts such as TFB~\cite{tfbbenchmark,qiu2024tfbcomprehensivefairbenchmarking}, in which $8068$ uni-variate time series and $25$ multi-variate time series datasets are included and used for benchmarking methods. A \textbf{unified} end-to-end evaluating pipeline is offered with standardized metrics and protocols, which has also motivated follow-up analyses and surveys on evaluation sensitivity and saturation.~\cite{kim2025comprehensivesurveydeeplearning,nochampions,wang2025accuracylawfuturedeep}

\textbf{Large Models.} Inspired from pretrained foundational models for NLP, foundational models for time series have been proposed. Early representatives include unified or decoder-only forecasting backbones trained on multi-source collections~\cite{unitsfoundationalmodel,googletimeseriesfoundationalmodel}. More recently, stronger 2025-era foundation models and training objectives have been explored, e.g., Sundial~\cite{sundial2025}, LightGTS~\cite{lightgts2025}, and UDE~\cite{ude2025}, as well as in-context fine-tuning for time-series foundation models.~\cite{icfttsfm2025} In parallel, a growing line of work explicitly investigates \emph{LLM+time series} hybrids: digit-tokenization / zero-shot extrapolation with off-the-shelf LLMs~\cite{llmtime2023}, prompt-based GPT-style forecasters~\cite{tempollm4ts,lstprompt2024,s2ipllm2024,patchinstruct2025}, ``reprogramming'' pretrained LLMs for forecasting~\cite{timellmllm4ts}, and alignment-style approaches that adapt LLM representations to numerical series~\cite{llm4ts}. More recently, integrated heterogeneous prompting and cross-modal alignment frameworks have also been explored~\cite{timeprompt2025}. Relatedly, language-model-like sequence backbones have also been explored for time series tasks beyond Transformers (e.g., RWKV-style models)~\cite{rwkvtsfoundationalmodel,foundationalmodelssurvey}.

\textbf{Previous Calls for Better Large Datasets and Better Practice.} Previously there have been calls for more scientific and better Large Datasets. For example, ~\cite{invitedtalkfoundamentallimitations} states the fact that such a foundational model might perform well on average on multiple datasets at the cost of poorer performance on certain subsets of datasets (similar to the James-Stein Paradox~\cite{jsparadox-1,jsparadox-2}). ~\cite{invitedtalkfoundamentallimitations,invitedtalkmultimodaltsm} also call for future time series datasets with multi-modality data, for the reason that time series alone (especially uni-variate time series) do not have enough information (e.g. electricity and market time series might have similar trend in input, but might evolve in very different ways, hence more information is needed). More recently, some work states other anomalies of the time series forecasting community: ~\cite{nochampions} shows the sensitivity of evaluation results of network architectures to hyperparameters, ~\cite{wang2025accuracylawfuturedeep} shows some forecasting benchmarks have been almost saturated.

In this paper, we analyze the root cause of these phenomena ground-based. We propose that these phenomena are caused by an incompatible conflict between Domain SOTA and Cross-Domain Generalizability, which is especially \textbf{irreconcilable for the general domain time series forecasting tasks the community has been working on}. Because of this irreconcilable conflict, current general-domain time series forecasting neural networks lag behind domain-specificially designed time series neural networks by enough margin that these domains develop their own SOTA neural networks (that are irrelevant to general domain time series neural networks).

\section{The Irreconcilable Conflict: Single-Domain vs. General-Domain SOTA for Time Series Neural Networks}

The central thesis of this paper is that any single architecture in Time Series Forecasting (TSF) faces an irreconcilable conflict between achieving domain-specific State-of-the-Art (SOTA) performance and maintaining cross-domain generalizability. While fields like NLP have converged toward a ``one-size-fits-all'' foundation model, TSF remains stubbornly fragmented. We argue that this is not a temporary limitation of current algorithms, but a fundamental conflict rooted in the nature of time series forecasting. Specifically, a model that generalizes must discard the unique contexts and structural priors that are essential for optimal performance in specialized fields. In the following sections, we demonstrate the irreconcilability of this conflict by analyzing two primary catalysts: the inherent cross-domain diversity of time series and the natural constraints of data scarcity that prevent traditional scaling.

\subsection{Cross-Domain Diversity}\label{sec: cross-domain diversity}

One leading cause of the conflict comes from the diversity among different time series domains.

First, Time series data is rarely ``self-contained''; its behavior \textbf{is usually driven by external ``context'' that varies in structure and modality}. For instance, weather forecasting relies on spatial-temporal context like geographic topography or atmospheric pressure grids, whereas financial forecasting depends on unstructured modalities like news sentiment or limit order book dynamics. Because it is architecturally difficult to build a single model capable of accepting such diverse auxiliary inputs, most existing works ignore these ``extra'' contexts to maintain a unified, sequence-only structure. Even for some ``Multimodal TSF Models'' which accept more modalities as input~\cite{neurips24talk}, they can not include all domain-specific modalities. For example, geographic topography for weather, stock market in another country for financial market prediction, etc., cannot be well-aligned as input to a single ``Multimodal TSF Model''. By treating all domains as simple numerical sequences, these models suffer from an inevitable Bayes error. They discard the domain-specific information necessary to reach the theoretical limit of predictability, trading precision for architectural uniformity.

Furthermore, time series from different domains are governed by \textbf{distinct mathematical and physical laws}, requiring different feature and structural priors. For example, stock prices and tradings in financial markets can be affected by Macroeconomic Indicators (e.g. policies, government reports)~\cite{chen1986economic}, Cross-Assets and Inter-Market Correlations (e.g. gold price rises when there is a fear in stock market)~\cite{longin2001correlation}, Corporate and Fundamental Events, etc; while a city's temperature can be affected by Geographic Baseline, Urban Structure, Anthropogenic Heat, Blue and Green Infrastructure, etc.~\cite{oke1982heat,stewart2012local} Domain experts design specific rules, features, neural network architectures and structures to take into account these factors. An architecture's specific design may become a liability in another domain. Consequently, a single architecture cannot be simultaneously optimal for the rigid cycles of physical systems and the stochastic volatility of human-driven markets.

\subsection{Data Scarcity} \label{sec: data scarcity}

While fields like Natural Language Processing (NLP) and Computer Vision (CV) have thrived by scaling model parameters alongside massive datasets, TSF faces a fundamental bottleneck in data availability. This scarcity is not merely a logistical hurdle or budget limitation, but a natural constraint that prevents a ``one-size-fits-all'' scaling approach.

\paragraph{Natural Scarcity and Augmentation Challenge}

Compared to the trillions of tokens available in language corpora, high-quality time series data is naturally scarce. Many critical domains, such as macroeconomics or long-term climate cycles, produce only a single data point per month or year. Furthermore, traditional data augmentation techniques used in CV (like flipping or cropping) often destroy the temporal dependencies and causal structures inherent in time series, making it difficult to synthetically expand training sets without introducing significant bias.

Furthermore, there is a significant imbalance in dataset sizes across different application domains. For example, the traffic and electricity domains each have nearly 200,000 total timestamps. In contrast, four of the ten domains, including prominent fields like stock and healthcare, have an aggregate dataset size of no more than 2,500 timestamps. Figure~\ref{fig: total timestamps across different fields} visualizes this disparity in total timestamps across different fields for TFB~\cite{tfbbenchmark}.

\begin{figure*}[!h]
  \includegraphics[width=0.9\textwidth]{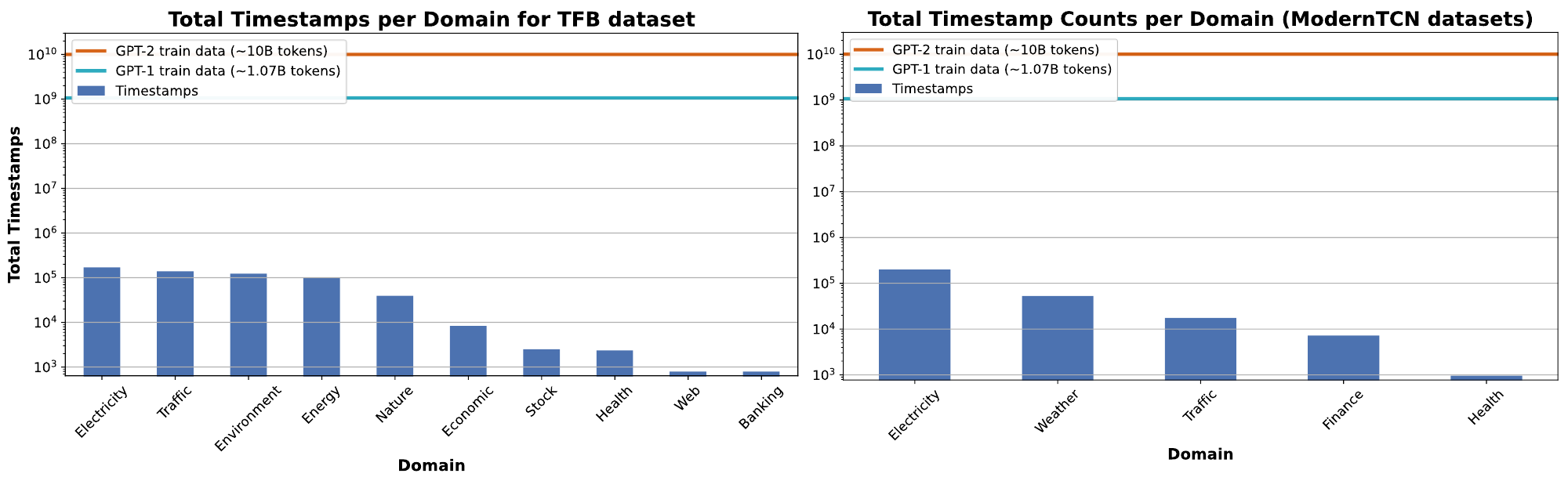}
  \centering
  \caption{Left: Number of total timestamps across different domains for multivariate forecasting datasets in TFB dataset~\cite{tfbbenchmark}; Right: Number of total timestamps across different domains forecasting tasks used by most TSF papers represented by ModernTCN~\cite{moderntcn}. All are small compared to typical NLP datasets~\cite{gpt}.}
  \label{fig: total timestamps across different fields}
\end{figure*}

\paragraph{The Temporal Boundary of Approximation Error}

More importantly, TSF is bounded by a unique temporal constraint: even if one increases sampling frequency or collects data from more sensors within a fixed time span $T$, the information gain is limited. As we demonstrate in Appendix~\ref{app: temporal boundary proof}, for a given observation window, the approximation error of the underlying process is fundamentally bounded by $O(1/\sqrt{T})$. Unlike other domains where data can be gathered in parallel across the internet, time series data is tethered to the linear progression of time. This creates an inevitable estimation error that cannot be overcome simply by increasing sample density.

\paragraph{Failure of Scaling and the Need of Human Priors}

The consequence of this data limitation is that TSF cannot rely on the ``scaling laws'' that define modern AI. In the absence of infinite data to ``wash away'' architectural inefficiencies, there is an \textbf{almost-inevitable estimation error} for any single domain. The only way to address it is by injecting domain-specific human priors into the model. For example, in power grid forecasting, incorporating physical constraints like Kirchhoff’s laws or seasonal load patterns allows a model to achieve high precision with limited data. However, these priors are strictly domain-dependent This necessity for specific inductive biases reinforces the irreconcilable gap between a unified, general-purpose architecture and a domain-specific SOTA model.

\section{Consequences of the conflict and focusing on a `General Domain Time Series Forecasting'} 

\subsection{Peculiar observations in Time Series Forecasting Community}

In our TSF community, there have been peculiar observations and warning signals. These observations can actually be explained starting from the irreconcilable conflict we proposed. We provide two representative examples.

\textbf{TSF Neural Networks do not always beat traditional methods.} One specific observation by previous work~\cite{M4competition,neurips24talk} is that traditional statistical forecasters (e.g., ARIMA~\cite{boxjenkins2015}, exponential smoothing/ETS~\cite{hyndman2002state}, and the Theta method~\cite{theta2000}) excel at general-domain time series forecasting tasks even compared with SOTA neural network methods. This is because traditional TSF methods are designed to fit general properties of general time series, and have been utilizing most of the features and properties shared across general domain time series. Some deep learning neural network methods are actually inspired by these traditional methods, for example, FITS~\cite{FITS2024} proposes an FFT-LowPassFilter-Linear-iFFT learning pipeline, which connects to classic adaptive/Wiener filtering ideas~\cite{wiener1949extrapolation,widrow1975adaptive} that date back more than half a century.

\textbf{Saturation of metrics of current general domain TSF neural network methods.} Another observation is that the newly proposed neural networks have been making less and less progress on benchmarks. For example,~\cite{wang2025accuracylawfuturedeep} proposed Accuracy Law, stating that accuracies for ETT, Electricity and Weather datasets are unlikely to be improved. As a result shown in~\cite{nochampions}, hyperparameters have great impact on general time series forecasting models' rankings: result variance is too large that almost all 8 general domain TSF neural networks tested can be made ``SOTA on an averaged metric'' with a ``properly set'' experiment setting, even on large, general-domain datasets that include many time series data samples such as TFB~\cite{tfbbenchmark} and Gift-eval~\cite{gifteval}. This indicates that simply increasing dataset size would not resolve the saturation issue. The devil lies in the inherent irreconcilable conflict between single domains performance and performance across general domains: a neural network designed targeting general time series performance can only leverage general domain features and properties, thus such direction can be easily exploited by hunderds of researcher years.

\subsection{Diverse SOTA proposed by different task Communities irrelevant to the time series community}

The gap between general-domain time series forecasting neural networks and domain-specific TS neural networks is so huge that, in most cases, single time series domain usually utilizes or develops their own SOTA methods. These methods are usually developed from general deep learning or machine learning methods, and unfortunately, are usually irrelevant to general domain time series forecasting neural networks, especially those developed in recent 3-4 years.

For example, in Table \ref{table: table for related competitions and challenges.} we list some single-domain time series forecasting challenges and competitions online. The top winning methods are based on, either statistical methods, or feature engineering and ML-based methods, or `plain' deep-learning methods (rather than DLTSF methods, which are architectures designed for general domain raw time series inputs).

\begin{table*}[!h]
\small
\begin{center}
\begin{tabular}{ccc}
\toprule
Forecasting Competitions &\#Teams& Top-1/Top-N Methods \\
\midrule
Kaggle \& Optiver:&3225&Top1: Feature Engineering +\\
Trading at Close&&Ensembled Model of CatBoost,GRU,Transformers~\cite{optivertradingatclose}\\
\midrule
Kaggle \& Jane Street:&4085&Top1: AutoEncoder+MLP+XGBoost~\cite{janestreetmarketprediction}\\
Market Prediction&&\\
\midrule
M6 Forecasting Competition&226&Top9: 1Judgement, 4TS-based, 3ML-based, 1NS~\cite{m6forecastingcompetition}\\
\bottomrule
\end{tabular}
\end{center}
\caption{\small Examples of Challenges and Competitions related to time series forecasting. \textbf{For Kaggle-Optiver Competition}: we further check top4 open-discussion solutions (all in top 20), which are all based on Feature Engineering and ensemble of traditional ML methods/basic DL methods. \textbf{For Kaggle-Jane Street Competition}: we further check top4 open-discussion solutions (all in top 20), which are all based on simple nets like MLPs, GBMs, etc. \textbf{For M6 competition}: Judgment: either pure judgmental or judgment-informed; TS-based: (traditional) time series approaches, but also their combinations; ML-based: ML approaches integrated with TS and combinations; NS: not specified. Though not specifically discussed by M6~\cite{m6forecastingcompetition} whether the teams use NNs proposed by time series forecasting community, but it seems most still use traditional deep-learning approaches, represented by~\cite{m6solution}.}
\label{table: table for related competitions and challenges.}
\end{table*}

Apart from Kaggle time series-related competitions, we also check the relevant academic work in domains that are related to time series, and appear in the datasets from ~\cite{Informer2021}. We selected 9 representative academic papers on domains including Finance, Weather, Health, Traffic from domain-top journals and conferences, and record the detailed ML/DL methods they use. Results are shown in Table \ref{table: table for related academic work.}. From the table we observe that: researchers in these domains develop their own SOTA methods that are built from basic building blocks (e.g. MLP, Transformers, etc.) and are not utilizing nor adapted from general domain time series methods.

\begin{table*}[h]
\small
\begin{center}
\begin{tabular}{cccc}
\toprule
Work & Published at & Domain & Method Used/Discussed \\
\midrule
~\cite{jfeml}~\cite{jfqamlbasedprediction}&JFE,JFQA&Finance&Features+MLP\\
\midrule
~\cite{panguweather}~\cite{gencast}&Nature&Weather&3D Transformer,DDPM\\
\midrule
~\cite{health2}~\cite{mira2025}&ML Conf.&Health&Foundation model,Transformer\\
\midrule
~\cite{traffic1}~\cite{traffic3}~\cite{robustlight2025}&ML Conf.&Traffic&Game Theory,Diffusion RL,Transformer\\
\bottomrule
\end{tabular}
\end{center}
\caption{Examples of academic work published at top journals and conferences, studying typical problems with deep learning in different domains, within the past 3 years. ML Conf. refers to ICML, NeurIPS and ICLR. Our time series community has designed general time series forecasting neural networks, and they are tested on related datasets like Weather, Traffic~\cite{Informer2021}; however, academic work in corresponding domains are not using these general time series forecasting neural nets designed by time series community, implying the capability of currently designed general-domain TSF neural networks.}
\label{table: table for related academic work.}
\end{table*}

\section{Possible Solutions}

Given the fundamental barriers of diversity and scarcity, achieving SOTA performance requires a strategic departure from universalism. We propose several solutions which may overcome the underlying limitation.

\subsection{Domain-specific Neural Network}

The most direct solution is the development of Domain-Specific Neural Networks (DSNNs)—architectures intentionally narrow in scope but deep in domain integration. Unlike general foundation models that prioritize zero-shot transfer, DSNNs are engineered to internalize the unique contexts, constraints, and data-generating mechanisms of a specific field. In practice, this usually means that the model is not only trained \emph{on} a domain, but also designed \emph{with} domain structure in mind. A few representative examples illustrate why such specialization often dominates real deployments:

\textbf{Physics-Informed Architectures (Weather/Climate).} In weather and climate forecasting, many important models~\cite{NeuralGCM, ClimODE, panguweather, fuxiweather, climax} embed physical structure (e.g., conservation laws, dynamical constraints, spatial fields) into the modeling pipeline, which can outperform generic sequence backbones that must infer these priors from limited observations.

\textbf{Microstructure- and Regime-Aware Models (Finance).} Financial time series are heavily non-stationary and driven by market microstructure and regime shifts~\cite{RHINE, gajamannage2023real, rahimikia2025revisiting}; prior work shows that explicitly designed modules for such properties can yield clear gains, whereas scaling generic architectures provides limited benefit~\cite{huang2024generative}.

\textbf{Task-Specific Efficiency Under Scarcity (Energy/Operations).} In many operational forecasting problems (e.g., load forecasting), supervised models tailored to the target horizon, covariates, and constraints can match or exceed large pretrained models under limited data budgets~\cite{tang2025time, wu2025temporal}.

These examples suggest a practical lesson: when domain context is available and the goal is domain-level SOTA, inductive bias and domain integration matter more than universal architectural novelty. DSNNs therefore serve as a realistic anchor for what practitioners optimize, and motivate system-level solutions (e.g., meta-learning) that can reuse such experts across domains.

\subsection{Meta-Learning: A Systemic Alternative to Unified Architectures}

Given that the conflict between single-domain precision and universal generalizability is irreconcilable within a single architecture, we argue that true cross-domain ``generalization'' should not be sought at the neural network architecture level, but at the meta-system level. Instead of a fixed model trying to absorb all domain nuances, we propose meta-learning frameworks that can dynamically discover and deploy the optimal architecture for a given task.

\paragraph{Agentic Meta-Learning: the `LLM Scientist'}

The most flexible approach to this problem is an agentic framework, often referred to as an `LLM Scientist'. Rather than acting as a forecaster itself, where it often struggles with numerical precision, the LLM acts as an orchestrator that analyzes the ``context'' and ``properties'' we identified in Section~\ref{sec: cross-domain diversity}. Works like~\cite{zhao2025timeseriesscientist, time-LLM, tang2025tsfllm} typically employ specialized agents to perform LLM-guided diagnostics on a dataset’s statistics, subsequently proposing targeted preprocessing and selecting the best-fit model architecture from a library of experts. By utilizing the reasoning capabilities of LLMs to ``plan'' the forecasting strategy, these systems can inject human-like domain knowledge, such as identifying whether a series is stationary or cyclical, without requiring the underlying numerical model to be universal. This solves the ``Bayes error'' by allowing the system to use the most context-appropriate tools for each unique domain.

\paragraph{Learning-Based Meta-Selection}
A more automated alternative is learning-based meta-learning, where a ``meta-model'' is trained to predict the performance of various base-learners across a wide array of datasets. Recent frameworks like~\cite{AutoForecast, METAFORS, fformpp} demonstrate this by mapping the statistical characteristics of a new, unseen time series to a ``performance matrix'' of hundreds of specialized forecasting models. By learning which architectures excel on specific data properties (e.g., high-frequency vs. seasonal), these systems can recommend a SOTA-level model for a new domain in a zero-shot or few-shot manner. Unlike traditional scaling, which attempts to build one massive model, this approach builds a meta-recommender that leverages the efficiency of domain-specific models, thereby bypassing the data scarcity bottleneck by focusing on ``learning to select'' rather than ``learning to forecast'' from scratch.

\paragraph{Meta-Learning for Automated Feature Selection}

In many high-stakes domains, engineered covariates and feature selection remain essential because raw temporal sequences often hide the true causal drivers. Yet manual feature design is costly and domain-specific. meta-learning offers a scalable alternative by learning, from prior tasks, which feature constructions tend to help under which data regimes.

For example, Meta-MSGL~\cite{Meta-MSGL} uses a meta-learning controller to score high-dimensional meta-features and estimate the marginal utility of feature groups across datasets. Related meta-learned feature-importance mechanisms have also been explored in clinical time-series settings such as EHR modeling~\cite{health1}. Together, these approaches can prioritize transformations (e.g., seasonal Fourier terms) or indicators (e.g., trend/volatility measures) consistent with a series’ global statistics, mitigating the Bayes error by tailoring the input space to the task without manual intervention.

\section{Alternative Views}
\label{sec: alternative views}

\subsection{The Universal Scaling Hypothesis}

A prominent alternative view is that the conflict between domain-specific SOTA and generalizability is not fundamental, but rather a symptom of insufficient model capacity and data volume~\cite{scalinglawnlp}. Proponents of this view argue that if we scale architectures (e.g., to billions of parameters) and training corpora (to trillions of time points) sufficiently, a ``super-model'' would eventually internalize the diverse properties and latent contexts of all domains~\cite{attentionisallyouneed,gpt3}. Under this hypothesis, the model would effectively learn to ``switch'' its internal logic based on the input signal, rendering domain-specific architectures obsolete.

\textbf{Our Response:} While the universal scaling hypothesis has proven transformative for NLP and CV~\cite{scalinglawnlp,alexnet}, we argue it is fundamentally hindered in TSF by the approximation error bound discussed in Section~\ref{sec: data scarcity}~\cite{kuznetsov2014generalization}. In text or images, `more data' often translates to higher semantic coverage; in time series, the dataset size is often limited naturally, and even if we manage to collect more data manually, the available data is strictly limited by the temporal span $T$. Scaling the model size without a corresponding linear increase in the temporal horizon does not resolve the $O(1/\sqrt{T})$ bottleneck; instead, it frequently exacerbates overfitting to domain-specific noise~\cite{ScalingLawTSF2024}. Moreover, as we have mentioned, for time series domain diversity is more irreconcilable compared to CV/NLP: different time series domains follow different or even contradictory evolving rules, hence adding more domains/more data might even hurt the performance~\cite{invitedtalkfoundamentallimitations,invitedtalkmultimodaltsm}.

\subsection{DL methods are more human-efficient}

A reasonable counterargument is that deep learning methods are \emph{human-efficient}: once a generic architecture and training pipeline are established, practitioners can often achieve strong performance on a new dataset with minimal manual modeling effort (e.g., no hand-crafted feature engineering or explicit physical constraints)~\cite{foundationalmodelssurvey}. From this perspective, the goal of building universal architectures is not necessarily to beat the best domain-expert systems in every field, but to lower the barrier to entry and provide a reliable ``default'' forecaster across many domains.

\textbf{Our Response.} We agree that general-purpose neural forecasting systems can be attractive as a low-effort baseline. However, in high-stakes domains, the objective is typically \emph{near-ceiling} performance under domain constraints and covariates; the limiting factor is usually domain context and valid inductive bias, not the backbone architecture. Thus, a universal sequence-only model can be human-efficient but information-inefficient~\cite{invitedtalkfoundamentallimitations}.

Moreover, meta-learning can be \emph{at least as human-efficient} while offering a more plausible path to higher performance: it automates the selection/adaptation of specialized experts (and their priors) from data, instead of forcing one backbone to fit all domains~\cite{AutoForecast,METAFORS}. Therefore, the human-efficiency argument strengthens our position: prioritize domain-specific models and meta-learning systems over further ``universal'' architecture engineering.

\subsection{General SOTA is also important}

Another alternative view is that it is misguided to compare against domain-specific SOTA at all: what the TSF community should optimize for is \emph{general} SOTA---a model that performs best \emph{on average} across a broad suite of datasets under a unified protocol~\cite{tfbbenchmark,gifteval}. Under this lens, incremental gains on cross-domain benchmarks are still meaningful, even if a domain expert can engineer a stronger system for a particular field. Moreover, for researchers attempting to work on a specific domain, they might try some simple, neat, effective general-domain methods, then try to build or get some more complicated, complex single-domain SOTA methods. It is very questionable whether a complicated, complex, hyperparameter-sensitive ``general-domain SOTA'' is helpful for these specific-domain researchers~\cite{nochampions}.

\textbf{Our Response.} We agree that strong cross-domain baselines and fair unified benchmarks are valuable---and that ``generalizability'' is a legitimate scientific goal. Our concern is that the current ``general SOTA'' game is increasingly dominated by evaluation sensitivity (hyperparameters, preprocessing, dataset curation) and yields diminishing practical returns: leaderboard gains on averaged benchmarks rarely translate to the settings practitioners care about, where domain context is available and success is measured against a domain-specific ceiling~\cite{nochampions}.

Meta-learning preserves generalizability (a unified interface and reuse across datasets) while better matching domain practice: rather than forcing one backbone to fit all domains, it learns to \emph{select/adapt} specialized pipelines and inductive biases~\cite{AutoForecast,METAFORS}. This shifts progress from chasing average-score gains to learning better rules for choosing the right expert.

\section{Conclusion} 

\begin{figure}[t]
\centering
\includegraphics[width=0.9\columnwidth]{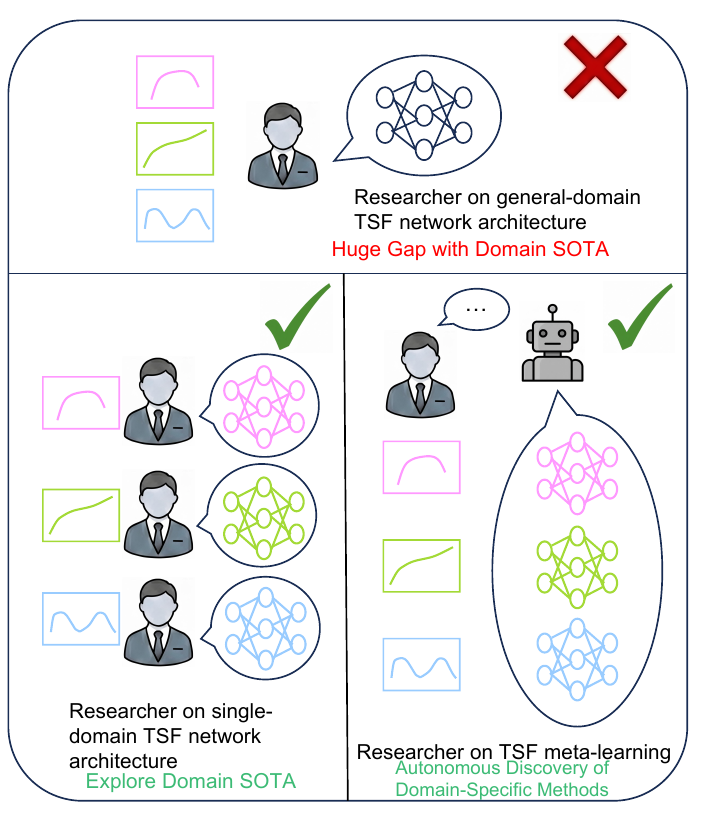}
\caption{Suggested practices: from pursuing a neural network architecture that has huge gaps with domain SOTAs to (1) developing single-domain SOTA neural network methods, or (2) developing meta-learning methods that are generalizable across domains.}
\label{fig:draft_practices}
\end{figure}

Time Series Forecasting is a critical task across industries, from weather prediction and electricity demand to quantitative trading. Motivated by these applications, the community has pursued neural architectures that achieve SOTA across broad, general-domain benchmarks. Yet, as we argue, an irreconcilable tension between specific-domain and general-domain SOTA means these models often fall short of domain-specific ceilings and are rarely used by practitioners who can leverage domain context and tailored inductive biases. Because researchers are not aware of this \textbf{fundamental, irreconcilable conflict} is previously overlooked (e.g. some opinion believes scaling up/better designs can solve the problem), substantial effort continues to go into general-domain architecture design.

We suggest that these efforts should be switched to the development of more useful and effective topics: as proposed in our work, we call for researchers to focus on \textbf{(1)} developing domain-specific time series forecasting neural networks or \textbf{(2)} developing meta-learning methods that are more generalizable than a certain neural network architecture.

\bibliography{example_paper}
\bibliographystyle{icml2026}

\newpage
\appendix
\onecolumn
\section{Proof for temporal boundary of estimation error} \label{app: temporal boundary proof}

This result is mainly based on ~\cite{kuznetsov2014generalization}, which bounded the generalization loss for non-stationary process. More specifically, let $((X_1,Y_1),\dots,(X_m,Y_m))$ be a sample of pairs from the distribution $\mathcal{Z} = \mathcal{X} \times \mathcal {Y}$, and $H = \{h: \mathcal{X} \rightarrow \mathcal{Y}\}$ be a class of hypothesis functions that admits some loss function $L: \mathcal{X} \times \mathcal{X} \rightarrow \mathbb{R}^+$. Assume data distribution will converge to some stationary distribution:
\begin{equation}
    \beta(a)=\sup_t\mathbb{E}{\left[\|\mathbf{P}_{t+a}(\cdot|\mathbf{Z}_{-\infty}^t)-\Pi\|_{TV}\right]} \rightarrow 0 \ \text{when} \ a \rightarrow +\infty
\end{equation}
where $\mathbf{Z}_a^b = (Z_a, \dots Z_b)$ denotes a vector of data samples, $\mathbf{P}_t$ denotes the distribution data sample at time $t$. Then we have the following statement:
\begin{theorem}\label{thm: lemma}
    Let $L$ be a loss function bounded by $M$ and $H$ be any set of hypotheses functions. Suppose $T=ma$ for some $m,a > 0$. Then for any $\delta > a(m - 1)\beta(a)$ with probability $1 - \delta$:
    \begin{equation}
        \mathcal{L}_{\Pi}(h)=\mathbb{E}_{\Pi}[\ell(h,Z)] \leq\frac{1}{T}\sum_{t=1}^T\ell(h,Z_t)+2\mathfrak{R}_m(H,\Pi)+M\sqrt{\frac{\log\frac{a}{\delta^{\prime}}}{2m}},
    \end{equation}
    holds for arbitrary hypothesis function $h \in H$, where $\delta^{\prime}=\delta-a(m-1)\mathsf{\beta}(a)$ and $\mathfrak{R}_m(H,\Pi)=\frac{1}{m}\mathbb{E}[\sup_{h\in H}\sum_{i=1}^m\sigma_i\ell(h,\widetilde{Z}_{\Pi,i})]$ with $\sigma_i$ a sequence of Rademacher random variables.
\end{theorem}
We define the series to be \emph{exponentially $\beta$-mixing} if $\beta(a) \leq Ce^{-da}$ for some positive constants $C,d$. Then, the argument in Section~\ref{sec: data scarcity} is actually a direct consequence of Theorem~\ref{thm: lemma}. A proof sketch is provided as follows;
\begin{proof}
    To maintain high probability, $\delta \rightarrow 0$, which gives $a(m - 1)\beta(a) \rightarrow 0$. This implies $T\beta(a) \rightarrow 0$. Thus, by the exponentially $\beta$-mixing condition, we have $a = O(\log T)$, then $\sqrt{\frac{\log\frac{a}{\delta^{\prime}}}{2m}} = \tilde{O}(1/\sqrt{T})$.
\end{proof}

\section{LLM Usage}
LLMs are used in this paper for polishing writings.

\end{document}